\documentclass{article}

\usepackage{aaai18} 
\usepackage{graphicx}
\usepackage{subfigure}
\usepackage{multirow}
\usepackage{amsthm, amsmath, amssymb, amsxtra, amsfonts}
\usepackage{todonotes}
\usepackage{url}

\title{On Looking for Local Expansion Invariants\\in Argumentation Semantics: a Preliminary Report}
\author{
	Stefano Bistarelli \and Francesco Santini\\
	Universit\`a degli Studi di Perugia, Italy\\
	stafano.bistarelli@unipg.it\\
	francesco.santini@unipg.it\\
	\And
	Carlo Taticchi\\
	Gran Sasso Science Institute, Italy\\
	carlo.taticchi@gssi.it
}

\newtheorem{definition}{Definition}
\newtheorem{proposition}{Proposition}
\newtheorem{theorem}{Theorem}
\newtheorem{corollary}{Corollary}
\newtheorem{remark}{Remark}
\newtheorem{example}{Example}

\newcommand{\lin}{{\fontfamily{lmss}\fontseries{b}\fontshape{n}\selectfont in}}
\newcommand{\lout}{{\fontfamily{lmss}\fontseries{b}\fontshape{n}\selectfont out}}
\newcommand{\lundec}{{\fontfamily{lmss}\fontseries{b}\fontshape{n}\selectfont undec}}

\newcommand{\lsin}[2]{{\fontfamily{pag}\fontseries{b}\fontshape{sc}\selectfont in$_{#1}(#2)$}}
\newcommand{\lsout}[2]{{\fontfamily{pag}\fontseries{b}\fontshape{sc}\selectfont out$_{#1}(#2)$}}
\newcommand{\lsundec}[2]{{\fontfamily{pag}\fontseries{b}\fontshape{sc}\selectfont undec$_{#1}(#2)$}}

\newcommand{\allf}{\mathcal{F_A}}

\newcommand{\af}{$G=\langle\mathcal{A},\mathcal{R}\rangle~$}

\newcommand{\quot}[1]{``#1''}

\newcommand{\fig}[1]{Fig.~\ref{fig:#1}}

\begin{document}
	
	\maketitle
	
	\begin{abstract}
		We study invariant local expansion operators for conflict-free and admissible sets in Abstract Argumentation Frameworks (AFs). Such operators are directly applied on AFs, and are invariant with respect to a chosen ``semantics'' (that is w.r.t.\ each of the conflict free/admissible set of arguments). Accordingly, we derive a definition of robustness for AFs in terms of the number of times such operators can be applied without producing any change in the chosen semantics.
	\end{abstract}
	
	\section{Introduction}
	An \emph{Abstract Argumentation Framework}~\cite{DBLP:journals/ai/Dung95} (\emph{AF}), is represented by a pair $\langle \mathcal{A}, \mathcal{R} \rangle$ consisting of a set of arguments and a binary relationship of attack defined among them. Given a framework, it is possible to examine the question on which set(s) of arguments can be accepted, hence collectively surviving the conflicts defined by $\mathcal{R}$. A very simple example of AF is $\langle \{a, b\}, \{(a, b), (b, a)\} \rangle$, where two arguments $a$ and $b$ attack each other. In this case, each of the two positions represented by either $\{a\}$ or $\{b\}$ can be intuitively valid.
	AFs can also provide a basis for handling the evolution of situations in which instances of particular problems undergo changes; variations on the underlying information can be interpreted as modifications in the corresponding graph. Such modifications can be performed through operations of addition or subtraction of nodes and edges in the AF. As one can expect, introducing these changes might lead to obtain different semantics for the considered AF. We can classify the operations it is possible to perform on a framework in two types: the ones that change the semantics of the system and the ones that do not. In this paper, we focus on this latter type of operations (which leave the semantics unchanged), and reducing to the case of addition (or subtraction) of an attack.
	
	Our aim is to study a set of local expansion~\cite{DBLP:journals/ai/Baumann12} operators with respect to which the semantics is not altered.
	Due to the dynamic nature of certain problems, settling for a solution (in a particular AF) could not be sufficient to guarantee a good outcome in case the problem evolves. In a dynamic setting, it may happen that new arguments can change the meaning (and the outcome) of the conversation itself. Think, for example, to a negotiation or a persuasion dialogue. With invariant operators at dispose, one could test and possibly \quot{enforce}~\cite{DBLP:conf/comma/BaumannB10} the strength of its position. Also, invariant operators could be successfully exploited for computing, in an efficient way, the semantics of an evolving AF.
	
	The main content of this paper is a the notion of \quot{robustness} for AFs. The main idea is that every argument (and set of arguments) is more or less suitable to undergo changes in the belief base~\cite{kern}. Robustness gives a measure of how many changes an AF can withstand before changing its semantics. 
	The main part the work we carried out, concerns the study of particular modifying operators for which the semantics is invariant: through such operators it is possible to bring changes in the structure of a framework without changing its meaning.
 	
	In particular, we study such operators for \emph{conflict-free} and \emph{admissible sets}~\cite{DBLP:journals/ai/Dung95}, since they are at the centre of any classical semantics and they have never been studied before in these terms. Differently from other works done in this direction (see the Related Work section), we consider how difficult is to modify the whole set of extensions instead of a single one, for instance as in \cite{DBLP:conf/tafa/RienstraST15}.
	
	This paper is structured as follows. We first summarise the necessary notions of Abstract Argumentation, by presenting extension-based~\cite{DBLP:journals/ai/Dung95} and \emph{labelling-based}~\cite{Caminada2006} semantics. We then describe the characteristics of the invariant operators we want to design and we define such operators for conflict-free and admissible sets. We conclude the paper by showing that our approach is novel w.r.t.\ the related work, and by providing conclusive thoughts and ideas about future work.
	
	
	
	\section{Abstract Argumentation Frameworks}
	By neglecting the internal structure of each argument (e.g., \emph{premises} and a \emph{claim}), the framework becomes ``abstract'', that is we are not interested in the meaning of arguments any more, but we just focus on their relations and we look for general properties. Hence, snapshots of such discussions can be caught by using \emph{Abstract Argumentation Frameworks}~\cite{DBLP:journals/ai/Dung95}, namely directed graphs that clearly show the exchange of opinions as attacks between arguments/nodes. Then, working with an AF means to identify subsets of nodes, called \emph{extensions}, which share certain properties, according to a given \emph{semantics}. Below, we give the fundamental definitions for AFs and extension-based semantics.
	
	\begin{definition}[Abstract Argumentation Framework~\cite{DBLP:journals/ai/Dung95}]
		An Abstract Argumentation Framework is a pair \af where $\mathcal{A}$ is a set of arguments and $\mathcal{R}$ is a binary relation on $\mathcal{A}$, i.e., $\mathcal{R} \subseteq \mathcal{A} \times \mathcal{A}$.
	\end{definition}
	
	We denote the set of all AFs with the set of argument $\mathcal{A}$ with $\mathcal{F_A}$. For two arguments $a,b \in \mathcal{A}$, $(a,b) \in \mathcal{R}$ represents an attack directed from $a$ against $b$. We can interchangeably use $a \rightarrow b$. Moreover, we say that a set of arguments $E \in \mathcal{A}$ attacks an argument $a$ if $a$ is attacked by an argument $b \in E$.
	
	\begin{definition}[Acceptable argument \cite{DBLP:journals/ai/Dung95}]
		An argument $a \in \mathcal{A}$ is acceptable with respect to $E \subseteq \mathcal{A}$ if and only if $\forall b \in \mathcal{A}$ s.t.\ $(b,a) \in \mathcal{R}$, $\exists c \in E$ s.t.\ $(c,b) \in \mathcal{R}$ and we say that $E$ defends $a$.
	\end{definition}
	
	Dung defines several semantics of acceptance for computing subsets of arguments, called extensions, that share characteristic properties. Respectively, $\mathit{cf}$, $\mathit{adm}$, $\mathit{com}$, $\mathit{stb}$, $\mathit{prf}$ and $\mathit{gde}$ stand for conflict-free, admissible, complete, stable, preferred, and grounded extensions.
	
	\begin{definition}[Extension-based semantics]
		Let \af $\in \mathcal{F_A}$ be an AF. A set $E \subseteq \mathcal{A}$ is conflict-free in $G$, denoted $E \in S_{cf}(G)$ if and only if there are no $a,b \in \mathcal{A}$ s.t.\ $(a,b) \in \mathcal{R}$. For $E \in S_{cf}(G)$ it holds that:
		
		\begin{itemize}
			\item $E \in S_{adm}(G)$ if each $a \in E$ is defended by E;
			\item $E \in S_{com}(G)$ if $E \in S_{adm}(G)$ and $\forall a \in \mathcal{A}$ defended by $E$, $a \in E$;
			\item $E \in S_{stb}(G)$ if $\forall a \in \mathcal{A} \setminus E $, $S$ attacks $a$;
			\item $E \in S_{prf}(G)$ if $E \in S_{adm}(G)$ and $\nexists E' \in S_{adm}(G)$ s.t.\ $E \subset E'$;
			\item $E = S_{gde}(G)$ if $E \in S_{com}(G)$ and $\nexists E' \in S_{com}(G)$ s.t.\ $E' \subset E$.
		\end{itemize}
		
	\end{definition}
	
	For the sake of consistency, In the following we will interchangeably use the terms semantics and sets (as conflict-free and admissible sets), annotating them simply with $S$: even if the results developed in this paper concern sets, the same methodology can be applied to semantics as well (see conclusions). Many of these semantics exploit the notion of defence in order to decide whether an argument is part of an extension or not. Such a phenomenon, for which an argument becomes accepted in some extension after being defended by another argument belonging to that extension, is known as \quot{reinstatement}. This mechanism plays a fundamental role when one needs to understand how a semantics changes in case an AF is modified. We exploit the notion of reinstatement labelling introduced in~\cite{Caminada2006}.
	
	\subsection{Reinstatement Labelling}\label{reinst}
	Reinstatement is the process through which a non-accepted argument $a$ of a framework becomes accepted w.r.t.\ a  certain extension when its attackers are defeated by $a$ itself or other arguments. Dung provides several approaches for dealing with reinstatement, like stable, preferred, complete, and grounded semantics. In \cite{Caminada2006}, Caminada strengthens Dung's theory by introducing  \quot{reinstatement labelling}, namely a function that partitions the set of arguments of an AF into three classes: \quot{in}, \quot{out}\ and \quot{undec}. An argument is labelled \lin\ if all its attackers are labelled \lout, and it is labelled \lout\ if it has at least an attacker which is labelled \lin. Otherwise, it is labelled \lundec.
	
	Caminada proved that reinstatement labelling coincides with Dung’s notion of complete extension. In the same way, further restrictions allow for obtaining stable, grounded and preferred semantic as well. Moreover, the same author specifies an additional semantics, called semi-stable, and provides a partial ordering between the various semantics. In particular, we know that: every stable extension is a semi-stable extension, every semi-stable extension is a preferred extension, every preferred extension is a complete extension, and every grounded extension is a complete extension. In Def.~\ref{def:reinstatement} we provide a formal definition of reinstatement labelling.
	
	\begin{figure}[]
		\centering
		\includegraphics[scale=0.14]{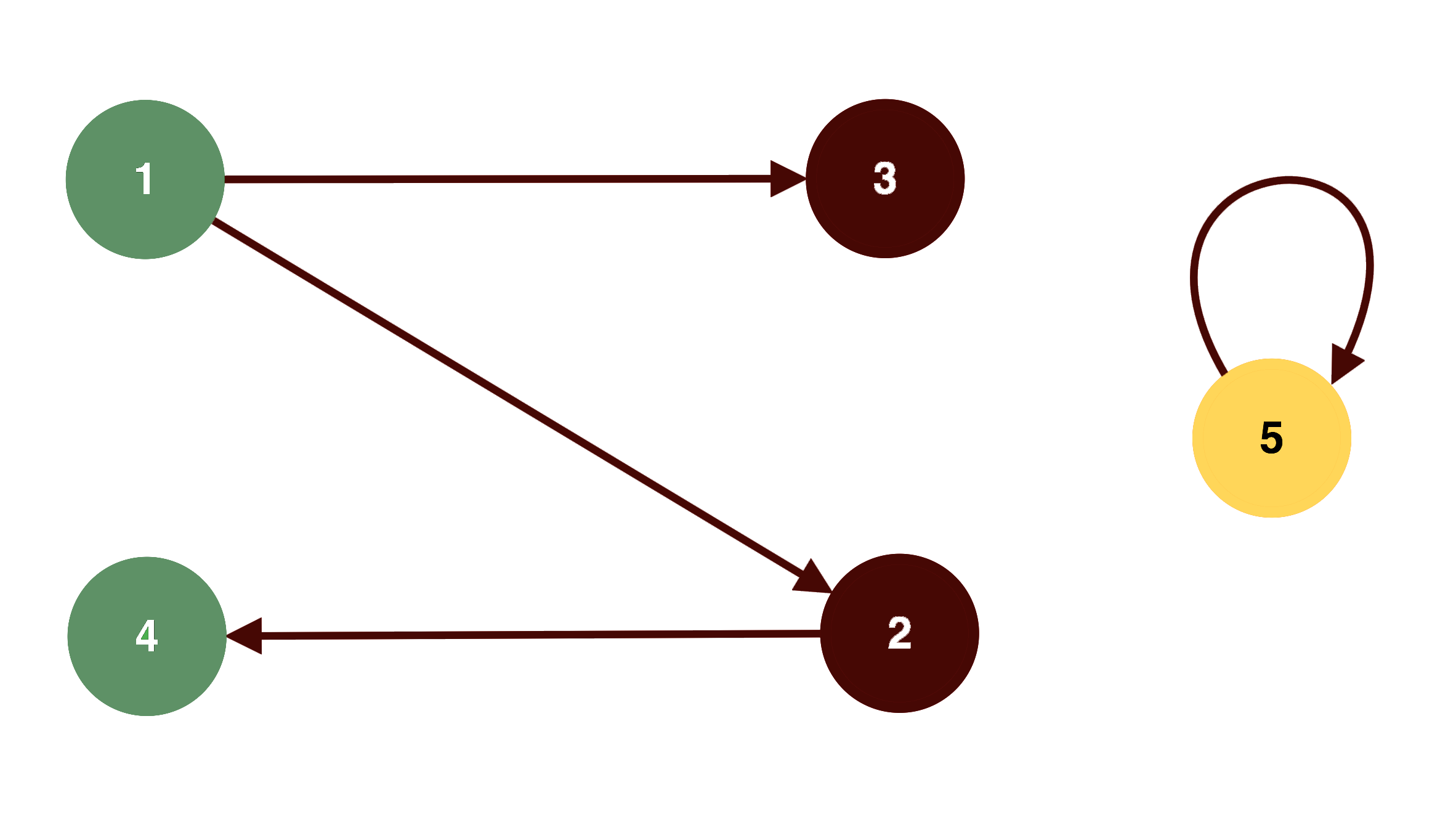}
		\caption{An AF-labelling in which $in(\mathcal{L}) = \{1,4\}$, $out(\mathcal{L}) = \{2,3\}$, and $undec(\mathcal{L}) = \{5\}$.}
		\label{fig:lab1}
	\end{figure}
	
	\begin{definition}[AF-labelling~\cite{Caminada2006}]\label{def:reinstatement}
		Let \af be a Dung-style argumentation framework. A labelling is a function $\mathcal{L} : \mathcal{A} \rightarrow \{in, out, undec\}$ such that $in(\mathcal{L}) \equiv \{a \in \mathcal{A}$ s.t.\ $\mathcal{L}(a) = in \}$, $out(\mathcal{L}) \equiv \{a \in \mathcal{A}$ s.t.\ $\mathcal{L}(a) = out\}$ and $undec(\mathcal{L}) \equiv \{a \in \mathcal{A}$ s.t.\ $\mathcal{L}(a) = undec\}$.
		We say that $\mathcal{L}$ is a reinstatement labelling if and only if it satisfies the following:
		
		\begin{itemize}
			\item $\forall a \in \mathcal{A}$ s.t.\ $a \in in(\mathcal{L}) $, $\forall b \in \mathcal{A}$ s.t.\ $(b,a) \in \mathcal{R}, b \in out(\mathcal{L}) $;
			\item $\forall a \in \mathcal{A}$ s.t.\ $a \in out(\mathcal{L}) $, $\exists b \in \mathcal{A}$ s.t.\ $(b,a) \in \mathcal{R}, b \in in(\mathcal{L}) $.
		\end{itemize}
		
	\end{definition}
	
	In \fig{lab1} we show an example of reinstatement labelling obtained with the web interface of ConArg\footnote{http://www.dmi.unipg.it/conarg}~\cite{DBLP:conf/clima/BistarelliRS14,DBLP:conf/comma/BistarelliRS16,DBLP:conf/at/BistarelliS12,DBLP:conf/ictai/BistarelliS11,DBLP:conf/tafa/BistarelliS11}: \lin\ ($\{1,4\}$), \lout\ ($\{2,3\}$) and \lundec\ ($\{5\}$) arguments are coloured in three different colours. Note that there exist more than one possible labelling for an AF. Moreover, there exists a connection between reinstatement labellings and the Dung-style semantics. This connection is summarised in Tab.~\ref{table:r_lab}.
	
	\begin{table}[]
		\centering
		\begin{tabular}{r l|c}
			\multicolumn{2}{c|}{Labelling restrictions} & Semantics   \\ \hline \hline
			\multicolumn{2}{c|}{no restrictions}        & complete    \\ \hline
			empty                 & \lundec              & stable      \\ \hline
			maximal               & \lin                 & preferred   \\ \hline
			maximal               & \lout                & preferred   \\ \hline
			maximal               & \lundec              & grounded   \\ \hline
			minimal               & \lin                 & grounded   \\ \hline
			minimal               & \lout                & grounded   \\ \hline
			minimal               & \lundec              & semi-stable \\ \hline
		\end{tabular}
		\caption{Reinstatement labelling vs semantics.}
		\label{table:r_lab}
	\end{table}

	\section{Invariant Operators}
	
	%
	A change in an AF can consist in addition (or subtraction) of nodes or edges. In this work we focus on modifications concerning attacks between arguments (in particular additions). This kind of transformation coincides with the notion of \emph{local expansion} (\cite{DBLP:journals/ai/Baumann12}).
	Introducing this type of changes in an AF may produce or not alterations on sets of extensions. This behaviour depends on two factors: the semantics we choose in computing the extensions and the change in the framework itself.
	
	After a modification, either a set of arguments is no more acceptable w.r.t.\ to a given semantics, or a new extension is generated, so the semantics of the AF will change in turn. On the contrary, if we consider the case in which extensions are preserved, further non trivial observations can be made for what concerns the semantics of the AF. For instance, even if the subsets of arguments remain unchanged, an admissible set can become also complete, if the right modifications are applied.
	Formally, an operator can be  defined as follows.
	
	\begin{definition}[Local expansion operator]
		Let $G=\langle\mathcal{A},\mathcal{R}\rangle \in \allf$ be an AF. A local expansion operator is a function $m : \allf \rightarrow \allf$ s.t.\ $m(G)=\langle\mathcal{A},m(\mathcal{R})\rangle$, where $m(\mathcal{R}) \supseteq \mathcal{R}$.
	\end{definition}
	
	If we consider those operators taking into account also Dung's semantics, we can study changes in the AFs from the point of view of sets of extensions. An explanatory example is provided in \fig{ver-or}: in the framework $G$, the extension $\{1,2\}$ is both admissible and complete while the extension containing argument $1$ is only admissible. After the modification $m=\{1 \rightarrow 2\}$ consisting in the addition of attack $1 \rightarrow 2$, the extension $\{1,2\}$ in the framework $m(G)$ is no longer admissible. On the other hand, after the change on the relations set, the extension $\{1\}$ that in $G$ was just admissible, also becomes complete in $m(G)$.
	
	\begin{figure}[]
		\centering
		\subfigure[$G$]
		{\includegraphics[width=2.5cm]{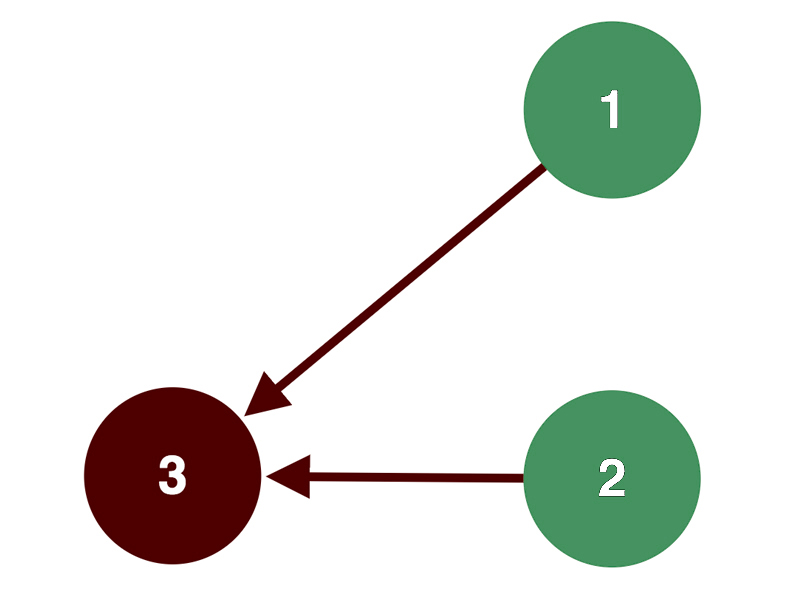}}
		\hspace{10mm}
		\subfigure[$m(G)$]
		{\includegraphics[width=2.5cm]{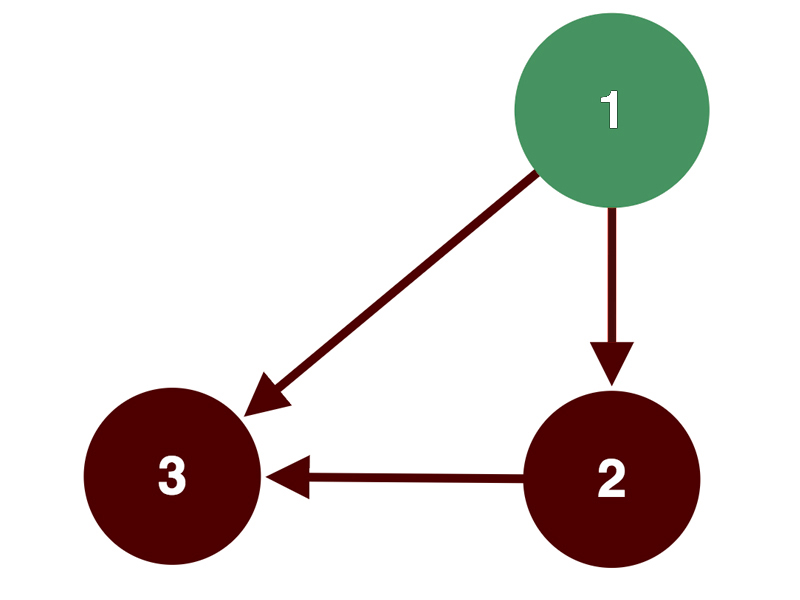}}
		\caption{On the left an AF $G$, on the right the same AF after have been modified by the operator $m$ adding the attack from argument 1 to argument 2. Argument $2$ becomes \lout\ from being \lin.}
		\label{fig:ver-or}
	\end{figure}
	
	\subsection{Semantics Equivalence}
	Our purpose is to find local expansion operators that leave the whole semantics unchanged, so instead of considering changes on the semantics induced by modifications on the graph, we look for the set of allowed changes that leave the semantics unmodified. In this way, we define semantics homomorphisms, namely operators with respect to which the semantics is invariant. In order to preserve the whole semantics, it is necessary to ensure that all the sets will not be modified, hence every set of extensions has to be, in turn, invariant with respect to these operators. We say that if two AFs have the same set of extensions w.r.t.\ a certain semantics, then the two frameworks are equivalent for such semantics.
	
	For this reason, we need the following definitions.
	
	\begin{definition}[Semantics inclusion] \label{seminc}
		Let $S$ and $S'$ be two sets of extensions. We say that $S \subseteq S'$ if and only if $~\forall E \in S ~\exists E' \in S'$ s.t.\ $E \subseteq E'$.
	\end{definition}
	
	\begin{definition}[$\sigma$-equivalence] \label{sequiv}
		Let $G$ and $G'$ be two AFs and $\sigma$ a semantics. We say that
		
		\begin{itemize}
			\item $G\equiv_{\sigma}G'$ if $S_\sigma(G) = S_\sigma(G')$;
			\item $G\sqsubseteq_{\sigma}G'$ if $S_\sigma(G) \subseteq S_\sigma(G')$.
		\end{itemize} 
	\end{definition}
	
	The equivalence we consider is referred as \emph{standard} in~\cite{DBLP:journals/ai/OikarinenW11}. Adding an attack in an AF can have different consequences. The most intuitive one is that the new attacked argument becomes defeated, and so it is forced to be removed from an extension. If we, instead, consider semantics in which the notion of acceptability is taken into account, defeating an argument could lead to accept another argument. In both the cases in which an argument become acceptable or is removed from an extension, the semantics would change. To distinguish the operators that reduce the set of extensions from those that expand it, we give the following definitions.
	
	\begin{definition}[Invariant operators]
		A local expansion operator $m$ is said non-decreasing w.r.t.\ the semantics $\sigma$ if for all the argumentation frameworks \af $\in \allf$, $G\sqsubseteq_{\sigma}m(G)$ and it is said non-increasing if $m(G)\sqsubseteq_{\sigma}G$. If $m$ is both non-decreasing and non-increasing, it is an \emph{invariant}: $G\equiv_{\sigma}m(G)$.
	\end{definition}
	
	The last case may occur when an attack has no effect on the set of extensions. Our purpose is exactly to find local expansion operators that guarantee this last outcome when adding an attack. In the following, an invariant operator will be referred to as $h$.
	
	It is necessary to understand how extensions react to changes in the AF. Since the main issue to deal with is due to the reinstatement, the idea we develop in order to define an invariant operator $m$ is to use the notion of reinstatement labelling. Once the arguments of the AF are labelled (with \lin, \lout\ or \lundec), there are nine ($3^2$) different ways an edge can be added among nodes, according to labels of the source and the target of the attack.
	
	\begin{definition}
		Let \af $\in \allf$ be an AF and $\sigma$ a semantics. The sets of arguments labelled \lin, \lout\ or \lundec\ in at least one extension are respectively:
		
		\begin{itemize}
			\item \lsin{\sigma}{G} $= \bigcup\limits_{S_\sigma(G)} \{a \in \mathcal{A}$ s.t.\ $a \in in(\mathcal{L})\}$;
			\item \lsout{\sigma}{G} $= \bigcup\limits_{S_\sigma(G)} \{a \in \mathcal{A}$ s.t.\ $a \in out(\mathcal{L})\}$;
			\item \lsundec{\sigma}{G} $= \bigcup\limits_{S_\sigma(G)} \{a \in \mathcal{A}$ s.t.\ $a \in undec(\mathcal{L})\}$.
		\end{itemize}
		
	\end{definition}
	
	Note that \lsin{\sigma}{G} coincides with the set of arguments of $G$ credulously accepted w.r.t.\ the semantics $\sigma$. In the following subsections we separately study the conflict-free  and the admissible semantics.
	
	\subsection{Operators for Conflict-Free Sets}
	The conflict-free property is very fragile: introducing a relation between two non conflicting nodes is sufficient to change the conflict-free sets. These sets can only be reduced: no new conflict-free set can be generated after the addition of an attack in the AF. Thus, every operator $m$ able to perform the addition of an edge in a graph $G$ produces another graph $m(G)$ in which the semantics is \quot{smaller} (in the sense that in some extensions of the new AF an argument disappears) or at most equal to the set deriving from $G$. $\mathcal{R}$ is the set of relations belonging to G, while $m(\mathcal{R})$ is the same set after the addition of an attack introduced by $m$. We avoid describing the trivial case in which $m(G) = G$, and we only consider the effective transformation of adding an attack (identical conclusions can be drawn in  case of subtraction).
	
	\begin{proposition} \label{th:cfni}
		Every local expansion operator $m$ is non-increasing w.r.t.\ the conflict-free sets.
	\end{proposition}
	
	\begin{proof}
		We have to show that $G \sqsupseteq_{cf} m(G)$ for every local expansion operator $m$, with \af $\in \allf$. This comes directly from the definition of conflict-free extension, since $m$ is such that $m(\mathcal{R}) \supseteq \mathcal{R}$.
	\end{proof}
	
	\begin{corollary}
		Any local expansion operator $m$  which is non-decreasing for conflict-free sets is also invariant: $G\equiv_{cf} m(G)$
	\end{corollary}
	
	\begin{proof}
		We know from Prop.~\ref{th:cfni} that every local expansion operator $m$ is non-increasing w.r.t.\ the conflict-free sets, that is $G \sqsupseteq_{cf} m(G)$. If $m$ is also non-decreasing, we have $G \sqsubseteq_{cf} m(G)$, and thus it is also invariant ($G\equiv_{cf} m(G)$).
	\end{proof}
	
	We conclude that an operator $m$ preserves the semantics only if it adds attacks between arguments which already were in conflict. We define an invariant operator $h$ for  conflict-free sets with the following theorem.
	
	\begin{theorem}[Invariant for conflict-free sets]
		Let \af $\in \allf$ be an AF. We have $G \equiv_{cf} h(G)$ if and only if $~\forall (a,b) \in h(\mathcal{R})$ 
		
		\begin{itemize}
			\item $(a,b) \vee  (b,a) \in \mathcal{R}$ or,
			\item $a \vee b \in \mathcal{A} \setminus $\lsin{\mathit{cf}}{G}.
		\end{itemize}
		
	\end{theorem}
	
	\begin{proof}
		We show that all conflict-free extensions are preserved if the above condition holds and vice versa.\\
		\quot{$\Longrightarrow$}:\\
		We have an $h$ such that $G \equiv_{cf} h(G)$. If the condition is not satisfied, then it would exists in $h(\mathcal{R})$ a relation between two arguments belonging to the same extension in $S_{cf}(G)$ and so $G \sqsupseteq_{cf} h(G)$. Contradiction.\\
		\quot{$\Longleftarrow$}:\\
		Suppose that the condition hold. If $(a,b) \vee  (b,a) \in \mathcal{R}$, then $a$ and $b$ are already in conflict and do not appear together in any conflict-free extension of $G$. Thus, no extension will be lost in $S_{cf}(h(G))$ when adding an attack between those arguments. On the other hand, having $a \vee b \in \mathcal{A} \smallsetminus$\lsin{cf}{G} means that one of the two arguments is never \lin\ and so there is no way to change the conflict-free extensions set by adding an attack between $a$ and $b$.
	\end{proof}
	
	If we take into account any of the semantics defined by Dung, we can conclude that adding an attack between two arguments belonging to a certain set always requires those arguments to be removed from such set, changing the semantics in turn. Hence, denying attacks between nodes within same set (which, then, do not attack each other in $G$) is a necessary condition in order to leave the semantics unchanged in $h(G)$.
	
	\subsection{Operators for Admissible Sets}
	Contrary to the conflict-free sets, for the admissible ones it is not possible to provide a theorem for the inclusion between semantics without taking into account reinstatement. Since arguments can be defended and consequently accepted with respect to a certain extension, we need to consider different types of interactions in order to find an operator capable of maintaining the semantics unchanged. Reinstatement labelling provides a powerful means to overcome the issue of comprehending how arguments defend each other inside an extension. Indeed, labellings are a more expressive way than extensions to suggest the acceptance of arguments. We exploit the notion of \lin, \lout\ and \lundec\ arguments to define the invariant operator $h$ for the admissible sets. In order to preserve this semantics, we have to guarantee that neither existent extensions will be destroyed, nor new one will be created. To achieve this, an operator $h$ has to ensure that extensions in the set remain conflict-free, \lin\ arguments are not defeated from outside and \lout\ and \lundec\ arguments do not become acceptable.
	We distinguish between modifications that reduce the semantics from modifications that enlarge it.
	
	\begin{theorem} \label{a_nd}
		Let \af\ be an AF. A local expansion operator $h$ is non-decreasing w.r.t.\ the admissible set if and only if $~\forall (a,b) \in h(\mathcal{R})$, there does exists a labelling $\mathcal{L}$ of $G$ such that
		
		\begin{itemize}
			\item $a,b \in in(\mathcal{L})$ or
			\item $a \in out(\mathcal{L}), b \in in(\mathcal{L})$, $(b,a) \notin \mathcal{R}$ and $\nexists c \in out(\mathcal{L})$ s.t.\ $(c,b) \in \mathcal{R}$ or
			\item $a \in undec(\mathcal{L}), b \in in(\mathcal{L})$.
		\end{itemize}
		
	\end{theorem}
	
	\begin{proof}
		We have to show that if $G\sqsubseteq_{adm}h(G)$ then the condition holds and vice versa.\\
		\quot{$\Longrightarrow$}:\\
		An operator $h$ is non-decreasing w.r.t.\ admissible sets. Suppose that there exists an attack relation $(a,b) \in h(\mathcal{R})$ such that in some labelling $\mathcal{L}$ of $G$ we have $a,b \in in(\mathcal{L})$ or $a \in undec(\mathcal{L}), b \in in(\mathcal{L})$. In both these cases, the admissible extension corresponding to the labelling $\mathcal{L}$ is lost in $h(G)$ and thus $G\sqsupseteq_{adm}h(G)$, so we have a contradiction. In the case that $a \in out(\mathcal{L})$ and $b \in in(\mathcal{L})$ and $(b,a) \notin \mathcal{R}$, if $(b,a) \notin \mathcal{R}$ and $\nexists c \in out(\mathcal{L})$ s.t.\ $(c,b) \in \mathcal{R}$, then the extension containing the only argument $b$ would be lost. Contradiction.\\
		\quot{$\Longleftarrow$}:\\
		If the condition holds, it is not possible that an extension in $S_{adm}(G)$ is also in $S_{adm}(h(G))$. Consider any labelling $\mathcal{L}$ of G. If $a$ is \lin\ or \lundec\ and $b$ is not \lin, then the addition of an attack $a \rightarrow b$ cannot make an admissible extension of $G$ to become unacceptable in $S_{adm}(h(G))$. If instead $a$ is \lout, it means that it is already defeated, so every argument $b$ belonging to some admissible extension of $G$ remains acceptable w.r.t.\ such extension also in $h(G)$. 
	\end{proof}
	
	\begin{theorem} \label{a_ni}
		Let \af\ be an AF. A local expansion operator $h$ is non-increasing w.r.t.\ the admissible set if and only if $~\forall(a,b) \in h(\mathcal{R})$, there does not exists a labelling $\mathcal{L}$ of $G$ such that
		
		\begin{itemize}
			\item $a,b \in in(\mathcal{L})$ and $\exists c \in out(\mathcal{L})$ s.t.\ $(a,c) \notin \mathcal{R}$ and $(b,c) \in \mathcal{R}$ or
			\item $a \in in(\mathcal{L}), b \in out(\mathcal{L})$ and $\exists c \in in(\mathcal{L})$ s.t.\ $(b,c) \in \mathcal{R}$ or
			\item $a \in in(\mathcal{L}), b \in undec(\mathcal{L})$ and $\exists c \in undec(\mathcal{L})$ s.t.\ $(c,c) \notin \mathcal{R}$ and $(b,c) \in \mathcal{R}$ or
			\item $a \in out(\mathcal{L})$, there is an odd length sequence of attacks from $b$ to $a$ and $\nexists c \neq b$ s.t.\ there is an odd length sequence of attacks from $c$ to $a$ but not from $a$ to $c$.
		\end{itemize}
		
	\end{theorem}
	
	\begin{proof}
		We show evidence that no new admissible extensions are generated for $G$ applying the operator $h$ if the conditions of the theorem are satisfied and vice versa.\\
		\quot{$\Longrightarrow$}:\\
		Suppose that $h(G)\sqsubseteq_{adm}G$. If there exists a labelling $\mathcal{L}$ for which $a,b \in in(\mathcal{L})$ and $\exists c \in out(\mathcal{L})$ s.t.\ $(a,c) \notin \mathcal{R}$ and $(b,c) \in \mathcal{R}$ then arguments $a$ and $c$ would become acceptable together, forming a new admissible extension. The same would happen whenever $a \in in(\mathcal{L}), b \in out(\mathcal{L})$ and $\exists c \in in(\mathcal{L})$ s.t.\ $(b,c) \in \mathcal{R}$ or in the case $a \in in(\mathcal{L}), b \in undec(\mathcal{L})$ and $\exists c \in undec(\mathcal{L})$ s.t.\ $(c,c) \notin \mathcal{R}$ and $(b,c) \in \mathcal{R}$. If instead the last condition does not hold, then $a$ would be defended from all the incoming attacks and so it would be accepted in some admissible extension of $h(G)$. In all these cases we reach a contradiction.\\
		\quot{$\Longleftarrow$}:\\
		We will see that if the conditions hold, it is not possible that a new admissible extension can be generated. For every labelling $\mathcal{L}$ of $G$, a non-increasing operator $h$ is allowed to add an attack between arguments $a$ and $b$ only in the following cases:
	
		\begin{enumerate}
			\item $a,b \in in(\mathcal{L})$ and $\nexists c \in out(\mathcal{L})$ s.t.\ $(a,c) \notin \mathcal{R}$ and $(b,c) \in \mathcal{R}$;
			\item $a \in in(\mathcal{L}), b \in out(\mathcal{L})$ and $\nexists c \in in(\mathcal{L})$ s.t.\ $(b,c) \in \mathcal{R}$;
			\item $a \in in(\mathcal{L}), b \in undec(\mathcal{L})$ and $\nexists c \in undec(\mathcal{L})$ s.t.\ $(c,c) \notin \mathcal{R}$ and $(b,c) \in \mathcal{R}$;
			\item \label{case4} $a \in out(\mathcal{L})$ and there is no odd length sequence of attacks from $b$ to $a$;
			\item \label{case5} $a \in out(\mathcal{L})$ and $\nexists c \neq b$ s.t.\ there is an odd length sequence of attacks from $c$ to $a$ but not from $a$ to $c$.
			\item $a \in undec(\mathcal{L})$.
		\end{enumerate}
	
		Case \ref{case4} means that $b$ is not responsible for $a$ being \lout, so the attack $a \rightarrow b$ is not sufficient to make $a$ acceptable in a new admissible extension. In case \ref{case5}, even if $a$ defeats $b$, it will not become admissible without also defeating $c$. In all remaining cases no arguments can be defended by $a$ (neither itself), so no new admissible extensions will form. 
	\end{proof}
	
	Given Th.~\ref{a_nd} and Th.~\ref{a_ni}, the following holds.
	\begin{corollary}\label{cor_adm}
		
		Let \af\ be an AF. A local expansion operator $h$ is invariant w.r.t.\ the admissible set, and we write $G\equiv_{adm} m(G)$, if and only if $~\forall(a,b) \in h(\mathcal{R})$, there does not exists a labelling $\mathcal{L}$ of $G$ such that
		
		\begin{itemize}
			\item $a,b \in in(\mathcal{L})$, or
			\item $a \in in(\mathcal{L}), b \in out(\mathcal{L})$ and $\exists c \in in(\mathcal{L})$ s.t.\ $(b,c) \in \mathcal{R}$, or
			\item $a \in in(\mathcal{L}), b \in undec(\mathcal{L})$ and $\exists c \in undec(\mathcal{L})$ s.t.\ $(c,c) \notin \mathcal{R}$ and $(b,c) \in \mathcal{R}$, or
			\item $a \in out(\mathcal{L}), b \in in(\mathcal{L})$, $(b,a) \notin \mathcal{R}$ and $\nexists c \in out(\mathcal{L})$ s.t.\ $(c,b) \in \mathcal{R}$, or
			\item $a \in out(\mathcal{L})$, there is an odd length sequence of attacks from $b$ to $a$ and $\nexists c \neq b$ s.t.\ there is an odd length sequence of attacks from $c$ to $a$ but not from $a$ to $c$, or
			\item $a \in undec(\mathcal{L}), b \in in(\mathcal{L})$.
		\end{itemize}

	\end{corollary}

	\begin{proof}
		The proof of this corollary is straightforward and comes from the proofs of Th.~\ref{a_nd} and Th.~\ref{a_ni}. In particular, if a labelling $\mathcal{L}$ of $G$ satisfying the properties above does not exists, then the local expansion operator $h$ is both non-decreasing (for Th.~\ref{a_nd}) and non-increasing (for Th.~\ref{a_ni}) w.r.t.\ the admissible semantics. Then $h$ is invariant w.r.t.\ the admissible semantics, because the modification on the set of relations does not allow any change in the semantics. Vice versa, if $h$ is invariant w.r.t. the admissible semantics, then $G \sqsubseteq_{adm} h(G)$ and $G \sqsupseteq_{adm} h(G)$  must hold. If a labelling exists such that at least one of the given properties is satisfied, then $h$ could be neither non-decreasing, nor non-increasing (or both), according to Th.~\ref{a_nd} and Th.~\ref{a_ni}. Thus, such a labelling can not exist.
	\end{proof}

	Below, we provide an example of how the conditions given in~\ref{cor_adm} allow to modify an AF without changing its semantics.
	
	\begin{example}
		Consider the AF in Fig.~\ref{fig:ex_cor}. The following attacks does not preserve the admissible semantics.
		\begin{itemize}
			\item $(1,4)$: since both 1 and 4 belong to $in(\mathcal{L})$, the extension $\{1,4\}$ would be deleted from $S_{adm}(G)$;
			\item $(4,2)$: 4 would defend 3 from 2, so the extension $\{3,4\}$ would be accepted in $S_{adm}(G)$;
			\item $(2,4)$: 4 does not attack 4 and it is not attacked by any other \lout\ node, so the extension $\{4\}$ would be deleted from $S_{adm}(G)$;
			\item $(2,1)$: the only odd length sequence of attacks toward 2 comes from 1, so 2 would defend itself from 1, generating the admissible extension $\{2\}$.
		\end{itemize}
	\end{example}

	\begin{figure}[b]
		\centering
		\includegraphics[width=6cm]{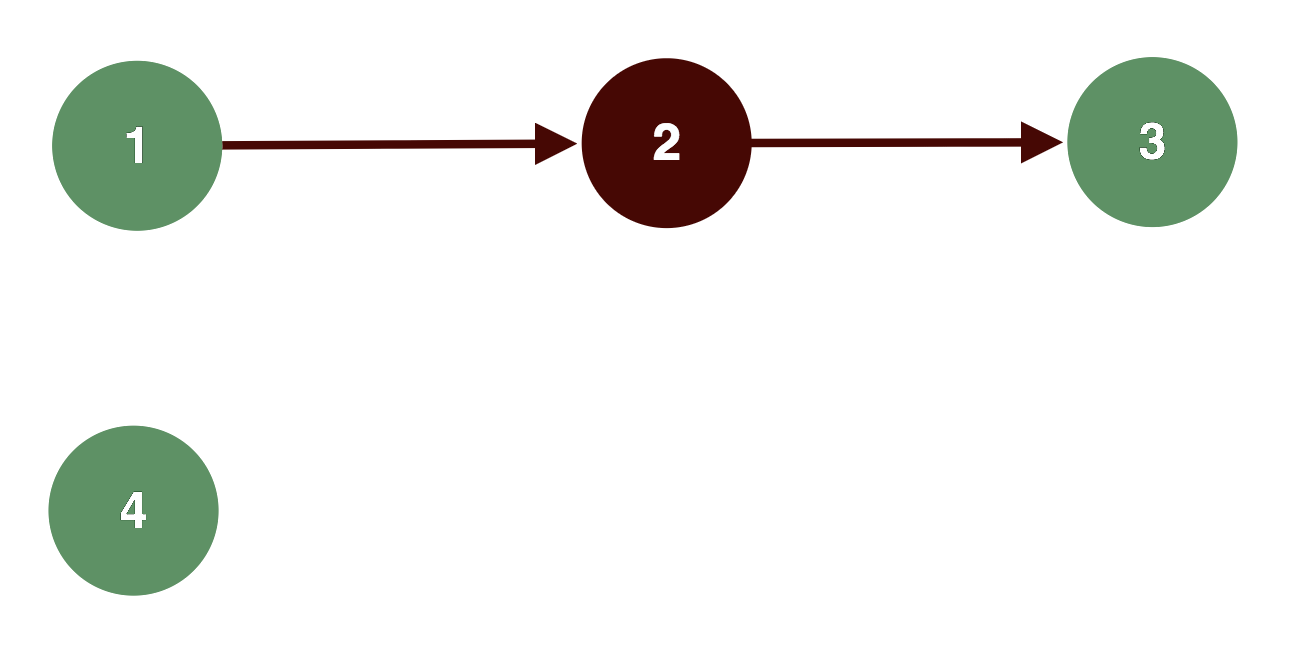}
		\caption{Example of an AF $G$ with $in(\mathcal{L}) = \{1,3,4\}$ and $S_{adm}(G)= \{ \{\},\{4\},\{1\},\{1,4\},\{1,3\},\{1,3,4\}\}$.}
		\label{fig:ex_cor}
	\end{figure}

	\begin{remark}
		In order to determine if an operator $m$ is invariant w.r.t.\ the admissible semantics, it is sufficient to consider only labelling in which $in(\mathcal{L})$ is maximal, that is the preferred extensions.
	\end{remark}
	
	In fact, in extensions that are not maximal, some arguments remain labelled \lundec\ even if they have different labels in more inclusive extensions (w.r.t.\ set inclusion). Thus, looking directly at the most inclusive extension allows for establishing rules able to preserve all the sets.
	
	Invariant operators can be used as a metric to measure the robustness of AFs. The idea is that, starting from $G$, different invariant operators can be applied in sequence, until no more $h$ exists for the last obtained AF: for example $h_4(h_3(h_2(h_1(G))))$ and no $h_5$ exists (as in Fig.~\ref{fig:ex_a}). Thus, the more operations are allowed for a framework, the more difficult it will be to change the extensions set for such semantics. We define the expansion-based robustness of a graph for a generic $\sigma$ as follows.
	
	\begin{figure}[]
		\centering
		\subfigure[$G$]
		{\includegraphics[width=2.7cm]{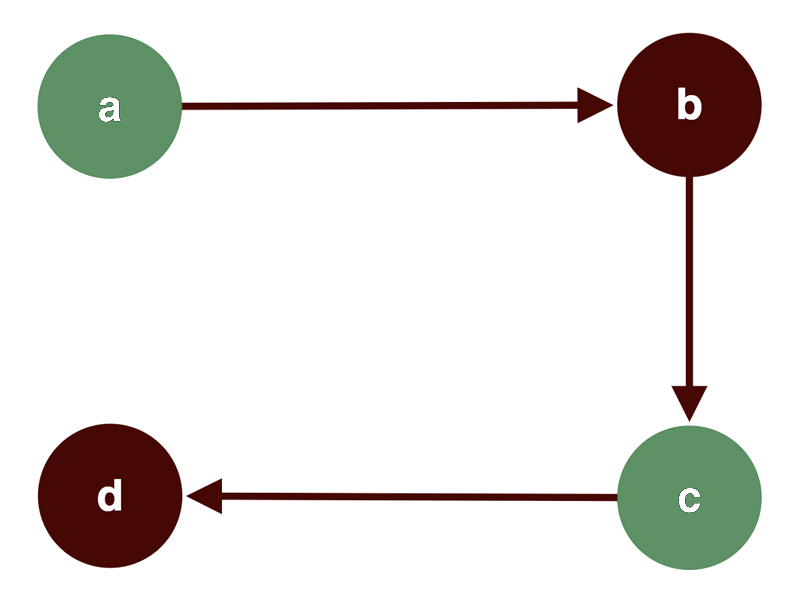}}
		\subfigure[$h(G)$]
		{\includegraphics[width=2.7cm]{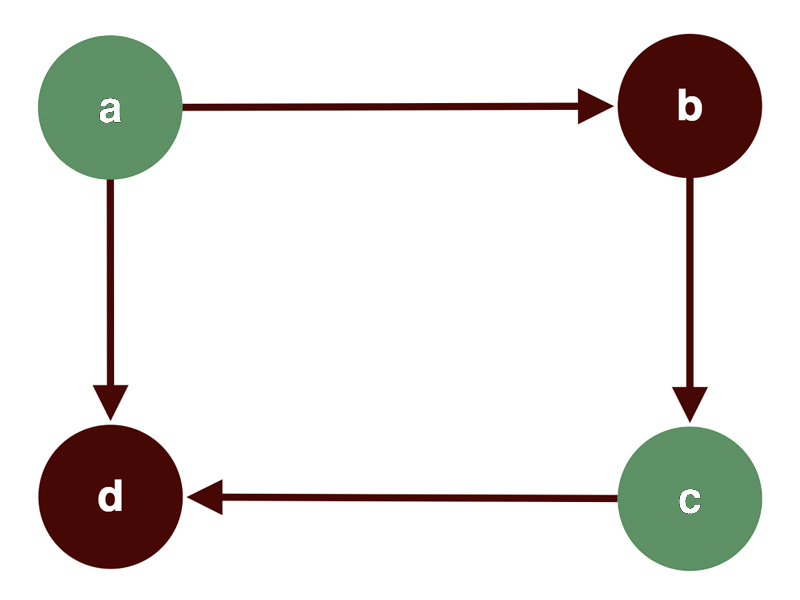}}
		\subfigure[$h(h(G))$]
		{\includegraphics[width=2.7cm]{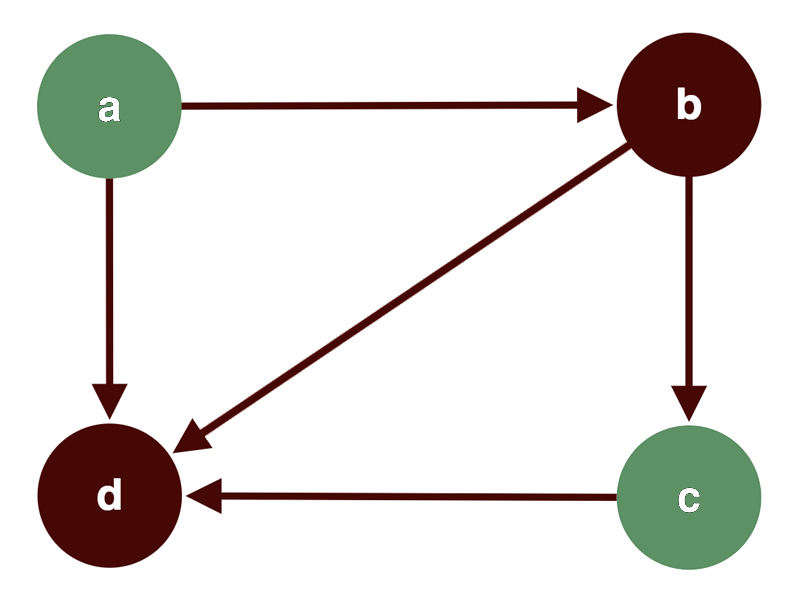}}
		\subfigure[$h(h(h(G)))$]
		{\includegraphics[width=2.7cm]{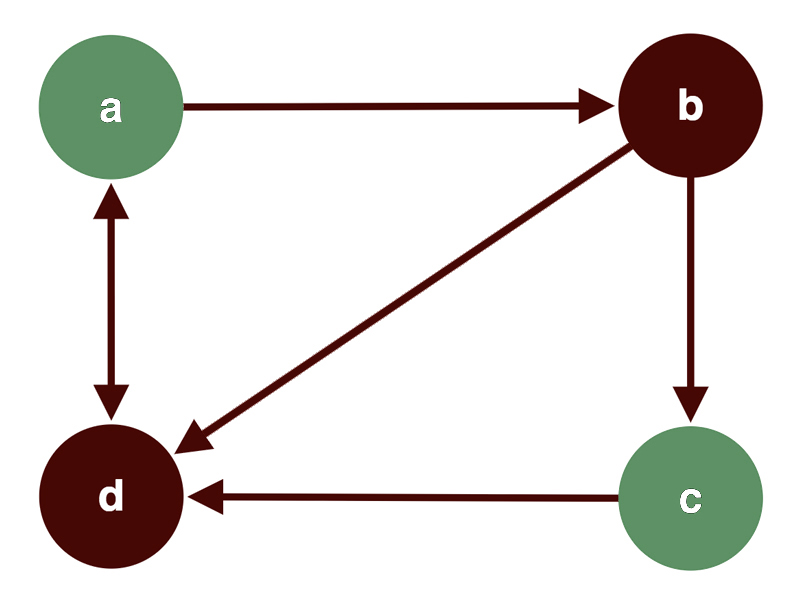}}
		\subfigure[$h(h(h(h(G))))$]
		{\includegraphics[width=2.7cm]{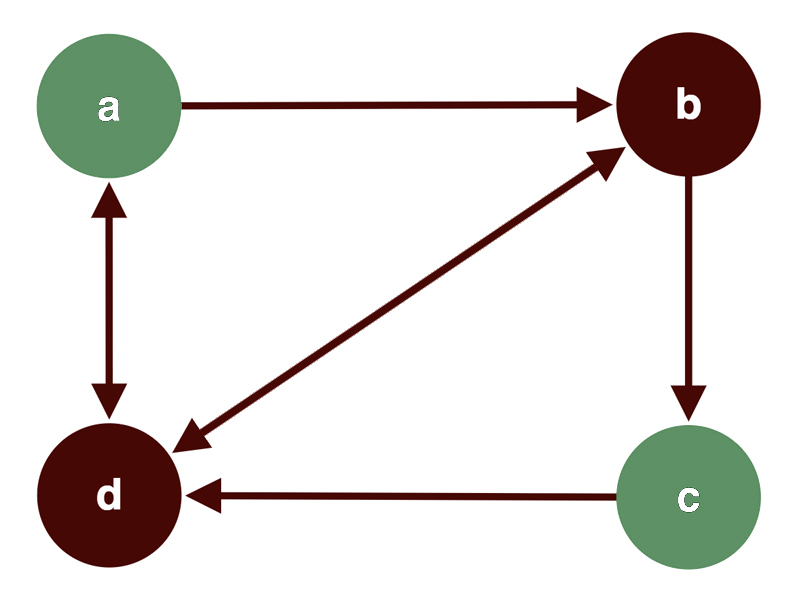}}
		\caption{From figure (a) to (e) the AFs obtained starting from $G$ and each time adding an attack through invariant operator $h$: $in(\mathcal{L}) = \{a,c\}$ persists.}
		\label{fig:ex_a}
	\end{figure}
	
	\begin{definition}[Local-expansion robustness]
		The local-expansion robustness degree of an AF \af w.r.t.\ a semantics $\sigma$ is measured as the maximum number $k$ of invariant operators $h_i$ that can be applied on $G$ s.t.\ $G\equiv_{\sigma} h_k(h_{k-1}(\dots(h_1(G)\dots))$.
	\end{definition}
	
	\section{Related Work}
	
	In the following we review the most meaningful works related to what presented in this paper.
	
	Rienstra et al. \cite{DBLP:conf/tafa/RienstraST15} focus on finding conditions under which the evaluation of an AF remains unchanged when an attack is added or removed. The authors  consider grounded, complete, preferred, stable and semi-stable semantics and, for each of them, a set of properties for which extensions are preserved is given. Those properties are in the form: \quot{given a certain labelling, attacks between two arguments with labels X and Y respectively are allowed (or not) for the semantics $\sigma$}. Invariance is intended w.r.t.\ a single extension and not w.r.t.\ the whole semantics (as we do). 
	
	The problem of finding principles stating whether an extension does not change after adding/removing an attack between two arguments is also addressed by Boella et al. in \cite{DBLP:conf/ecsqaru/BoellaKT09} and \cite{DBLP:conf/argmas/BoellaKT09}. Differently from us, the authors consider only the case in which the semantics of an AF contains exactly one extension, using the grounded semantics as example. 
	
	Cayrol et al \cite{DBLP:journals/jair/CayrolSL10} studied the impact on the evaluation of an AF when new arguments and attacks are added. They define a number of properties for the change operations according to how the extensions are modified. For instance, a change operation can be \quot{conservative} if the set of extensions is the same after a change. 
	
	In \cite{cayrol} is addressed the problem of revising AFs when a new argument is added. In particular, they focus on the impact of new arguments on the set of initial extensions, introducing various kinds of revision operators that have different effects on the semantics. For instance, \emph{Decisive revision}, as its name suggests, allows for making a decision by providing a revised extensions set with a unique non-empty extension.
	
	In \cite{DBLP:conf/comma/BaumannB10} the problem of revising argumentation frameworks according to acquisition of new knowledge is taken into account. 
	While attacks among the old arguments remain unchanged, new arguments and attacks among them can be added. In particular, the authors introduce the notion of \emph{enforcing}, namely the process of modifying an AF (and possibly changing its semantics) in order to obtain a desired set of extensions. This notion departs from our work, in which we instead look for operations that leave the semantics unchanged. 
	
	Also Baumann introduces the concepts of update and deletion~\cite{DBLP:conf/ecai/Baumann14}, focusing on modifications that retract arguments and attacks form an AF. New notions of equivalence are characterized through the so called kernels, namely functions that delete redundant attacks from a given framework. We instead concentrate on devising operators that permit both to modify AFs without changing their semantics, and to give a measure of how robust is a given AF, w.r.t.\ changes on the attack relations set.
	
	The concept of \emph{desire set} is also studied by Boella et al. in \cite{boella} with a work on persuasion in multi-agent systems, addressing the problem of choosing arguments to add into a system in order to maximise their acceptability w.r.t.\ the receiving agent. To this purpose, the notion of \quot{more appealing} argument is introduced: in making the choice of a belief to add, an argument is more appealing than another if it does not interact with previous goals and beliefs of the agent. This is in contrast with our goal that consists in keeping unaltered the set of extension. 
	
	The authors of \cite{DBLP:conf/tafa/CroitoruK13} show that every AF can be augmented in a normal form preserving the semantic properties. In such normal form no argument attacks a conflicting pair of arguments. A $\sigma$-\emph{augmentation} is an alteration of an AF that leaves unchanged the semantics $\sigma$. The changes in the AF can involve arguments (the only allowed operation is the addition) and attacks. 
	Due to the process through which is obtained, the normal form there proposed is not a representative for the isomorphism classes we use to construct our lattice. 
	
	A different, more restrictive, kind of equivalence is introduced in~\cite{DBLP:journals/ai/OikarinenW11}: two AFs $G$ and $G'$ are considered strongly equivalent to each other when they are equivalent after the conjunction with a third AF $H$. Since our intent is to provide a method for building equivalent AFs through the addition/deletion of attacks on a same framework, the notion of standard equivalence results to be more fitting than the strong equivalence.
	
	\section{Conclusions and Future Work}
	
	We defined invariant operators for AFs w.r.t.\ the semantics: these operators allow for performing changes on AFs while preserving the semantics. In particular, we have defined two operators, one for the conflict-free and one for the admissible sets, which can be applied to AFs for adding attack relations without resulting in changes to the set of extensions. The operators we have introduced exploit the notion of reinstatement labelling, and thus can be applied without even being aware of the extensions admitted for a given semantics. Moreover we gave the definition for the semantic equivalence between AFs, and we presented a method for computing the expansion-based robustness degree of a framework.
	
	Our study  has a very wide set of future perspectives. First, we plan to design invariant operators w.r.t.\ the complete, stable, semi-stable, preferred and grounded semantics (until now studied only w.r.t.\ single extensions~\cite{DBLP:conf/tafa/RienstraST15}).
	We would like to find the sets of arguments which are essential to preserve the whole semantics. Every change inside those sets modifies the semantics, while changes outside do not cause any alteration. By removing the non-core part of AFs, it is possible to obtain equivalent frameworks for which the computation of extensions is faster, especially for checking credulous/sceptical acceptance of arguments. 
	
	As further work, different notions of equivalence, e.g.\ local equivalence~\cite{DBLP:journals/ai/OikarinenW11}, could be taken into account, and additional modifications of AFs could be considered, as the deletion of attack or the addition/removal of arguments. We also plan to devise a more  general notion of robustness, involving the new modifications proposed above.
	By relaxing the conditions underlying invariant operators, and thus allowing the semantics to change, other operators could be obtained, that allow to reach ``compromises'': if two parts of a debate desire two different outcomes in terms of semantics, a compromise can be reached as a third semantics, that is the closest one w.r.t.\ those desired by both the counterparts. Definitions of closeness could be devised as well.
	Finally, we want to study local expansion operators also for semiring-based weighted AFs~\cite{DBLP:journals/ijar/BistarelliRS18,DBLP:conf/flairs/BistarelliRS16,DBLP:conf/csclp/BistarelliPS09,DBLP:conf/ecai/BistarelliS10}.
	
	\bibliographystyle{aaai}
	\bibliography{references}
	
\end{document}